\documentclass[letterpaper, 10 pt, conference]{ieeeconf}  

\IEEEoverridecommandlockouts                              

\usepackage{graphicx}
\usepackage{epsfig} 
\usepackage{mathptmx}
\usepackage{amsmath} 
\usepackage{bm}

\usepackage{amssymb}  
\usepackage{amsthm}
\usepackage{wasysym}
\usepackage[mathcal]{euscript}
\usepackage{bbm}
\usepackage{color}
\usepackage{lipsum}
\usepackage[ruled,vlined,linesnumbered]{algorithm2e}
\usepackage{hyperref}
\usepackage{etoolbox}
\usepackage[table,xcdraw,dvipsnames]{xcolor}
\usepackage{multirow}
\usepackage{bbm}

\usepackage{outlines}
\let\labelindent\relax
\usepackage{enumitem}
\setenumerate[2]{label=\alph*.}
\setenumerate[3]{label=\roman*.}

\makeatletter
\patchcmd{\@makecaption}
  {\scshape}
  {}
  {}
  {}
\makeatother

\newtheorem{defn}{Definition}

\newtheorem{prop}[defn]{Proposition}

\providecommand{\R}{\ensuremath \mathbb{R}}
\providecommand{\N}{\ensuremath \mathbb{N}}
\newcommand{\weights}{\mathcal{W}}
\newcommand{\biases}{\mathcal{B}}
\newcommand{\Zin}{Z}
\newcommand{\colval}{v}
\newcommand{\inset}{X\idxo}
\newcommand{\usset}{X\unsf}
\newcommand{\outset}{X\idxdepth}
\newcommand{\inzono}{Z\idxo}
\newcommand{\outzono}{Z\idxdepth}
\newcommand{\uszono}{Z\unsf}

\newcommand{\regtext}[1]{\mathrm{\textnormal{#1}}}

\newcommand{\defemph}[1]{\emph{#1}}
\newcommand{\ts}[1]{\textsuperscript{#1}}
\newcommand{\vc}[1]{\mathbf{#1}}

\newcommand{\opt}{^*}

\newcommand{\zeros}{\mathbf{0}}
\newcommand{\ones}{\mathbf{1}}
\newcommand{\eye}{\mathbf{I}}

\newcommand{\Wt}{\vc{W}}
\newcommand{\bias}{\vc{w}}
\newcommand{\xv}{\vc{x}}
\newcommand{\yv}{\vc{y}}
\newcommand{\ctr}{\vc{c}}
\newcommand{\ntuple}{\vc{u}}
\newcommand{\Gen}{\vc{G}}

\newcommand{\Acon}{\vc{A}}
\newcommand{\bcon}{\vc{b}}
\newcommand{\coef}{\vc{z}}

\newcommand{\slackeq}{\vc{n}}
\newcommand{\slackineq}{\vc{m}}

\newcommand{\norm}[1]{\left\Vert#1\right\Vert}

\newcommand{\diag}[1]{\regtext{diag}\!\left(#1\right)}
\newcommand{\union}{\bigcup}

\newcommand{\trans}{^\top}

\newcommand{\conzono}[1]{\mathcal{CZ}\!\left(#1\right)}

\newcommand{\relu}[1]{\varrho\!\left(#1\right)}
\newcommand{\nn}[1]{\vc{n}\!\left(#1\right)}

\newcommand{\linlayer}[1]{\mathcal{L}\!\left(#1\right)}
\newcommand{\Lobj}[1]{\ell_{\regtext{obj}}\!\left(#1\right)}
\newcommand{\Lcon}[1]{\ell_{\regtext{con}}\!\left(#1\right)}

\newcommand{\unsf}{_\regtext{unsf}}

\newcommand{\ncon}{{n_{\regtext{con}}}}
\newcommand{\ngen}{{n_{\regtext{gen}}}}
\newcommand{\nout}{{n_{\regtext{out}}}}
\newcommand{\ndata}{{n_{\regtext{data}}}}
\newcommand{\depth}{d}

\newcommand{\idx}[1]{^{(#1)}} 
\newcommand{\idxi}{\idx{i}}
\newcommand{\idxo}{\idx{0}}

\newcommand{\idxdepth}{\idx{\depth}}

\usepackage[
    style=ieee,
    doi=false,
    isbn=false,
    url=false,
    eprint=false,
    backend=biber,
    natbib=true
    ]{biblatex}
    
\bibliography{references}

\title{\LARGE \bf
Constrained Feedforward Neural Network Training\\
via Reachability Analysis
}

\author{Long Kiu Chung*, Adam Dai*, Derek Knowles, Shreyas Kousik, and Grace X. Gao
\thanks{* indicates equal contribution.
All authors are with Stanford University, Stanford, CA.
L.K. Chung is with the Department of Mechanical Engineering. 
A. Dai is with the Department of Electrical Engineering. 
D. Knowles, S. Kousik, and G.X. Gao are with the Department of Aeronautics and Astronautics. 
Corresponding author: \texttt{gracegao@stanford.edu}.
}
}

\begin{document}

\maketitle
\thispagestyle{plain}
\pagestyle{plain} 

\begin{abstract}
Neural networks have recently become popular for a wide variety of uses, but have seen limited application in safety-critical domains such as robotics near and around humans.
This is because it remains an open challenge to train a neural network to obey safety constraints.
Most existing safety-related methods only seek to verify that already-trained networks obey constraints, requiring alternating training and verification.
Instead, this work proposes a \textit{constrained} method to simultaneously train and verify a feedforward neural network with rectified linear unit (ReLU) nonlinearities.
Constraints are enforced by computing the network's output-space reachable set and ensuring that it does not intersect with unsafe sets; training is achieved by formulating a novel collision-check loss function between the reachable set and unsafe portions of the output space.
The reachable and unsafe sets are represented by constrained zonotopes, a convex polytope representation that enables differentiable collision checking.
The proposed method is demonstrated successfully on a network with one nonlinearity layer and $\approx 50$ parameters.
\end{abstract}

\section{Introduction}

Neural networks are a popular method for approximating nonlinear functions, with increasing applications in the field of human-robot interactions.
For example, the kinematics of many elder-care robots \cite{xiong2007development, ko2017neural}, rehabilitation robots \cite{xu2009adaptive, hussain2013adaptive}, industrial robot manipulators \cite{gribovskaya2011motion}, and automated driving systems \cite{tran2020nnv, shengbo2019key} are controlled by neural networks.
Thus, verifying the \textit{safety} of the neural networks in these systems, before deployment near humans, is crucial in avoiding injuries and accidents.
However, it remains an active area of research to ensure the output of a neural network satisfies user-specified constraints and requirements.
In this short paper, we take preliminary steps towards safety via constrained training by representing constraints as a collision check between the reachable set of a neural network and unsafe sets in its output space.


\subsection{Related Work}

Many different solutions have been proposed for the \textit{verification} problem, with set-based reachability analysis being the most common for an uncertain set of inputs \cite{liu2019algorithms}.
Depending on one's choice of representation, the predicted output is either exact (e.g. star set \cite{tran2019star, tran2020nnv}, ImageStar \cite{tran2020verification}) or an over-approximation (e.g. zonotope \cite{althoff2010reachability}) of the actual output set.
Reachability is most commonly computed layer-by-layer, though methods have been proposed that speed up verification by, e.g., using an anytime algorithm to return unsafe cells while enumerating polyhedral cells in the input space \cite{vincent2020reachable}, or recursively partitioning the input set via shadow prices \cite{rubies2019fast}.

Verification techniques have several drawbacks.
First, they do not provide feedback about constraints during training, so one must alternate training and verification until desired properties have been achieved.
Furthermore, verification by over-approximation can often be inconclusive, while exact verification can be expensive to compute.

Several alternative approaches have therefore been proposed.
For example, \cite{huang2015bidirectional} employs a constrained optimization layer to use the output of the network as a potential function for optimization while enforcing constraints.
Similarly, \cite{stewart2017label, xu2018semantic} adds a constraint violation penalty to the objective loss function and penalizes violation of the constraint.
These methods augment their networks with constrained optimization, but are unable to guarantee constraint satisfaction upon convergence of the training.
Alternatively, \cite{cruz2021safe} uses a systematic process of small changes to conform a ``mostly-correct'' network to constraints.
However the method only works for networks with a Two-Level Lattice (TLL) architecture, requires an already-trained network, and again does not guarantee a provably safe solution.
Finally, \cite{markolf2021polytopic} attempts to learn the optimal cost-to-go for the Hamilton–Jacobi–Bellman (HJB) equation, while subjected to constraints on the output of the neural network controller.
Yet, it does not actually involve any network training and is unable to handle uncertain input sets.

Recently, constrained zonotopes have been introduced as a set-based representation that is closed under linear transformations and can exactly represent any convex polytope \cite{scott2016constrained, raghuraman2020set}. 
Importantly, these sets are well-suited for reachability analysis due to analytical, efficient methods for computing Minkowski sums, intersections, and collision checks; in particular, collision-checking only requires solving a linear program.
We leverage these properties to enable our contributions.

\begin{figure}[t]
    \centering
    \includegraphics[width=0.99\columnwidth]{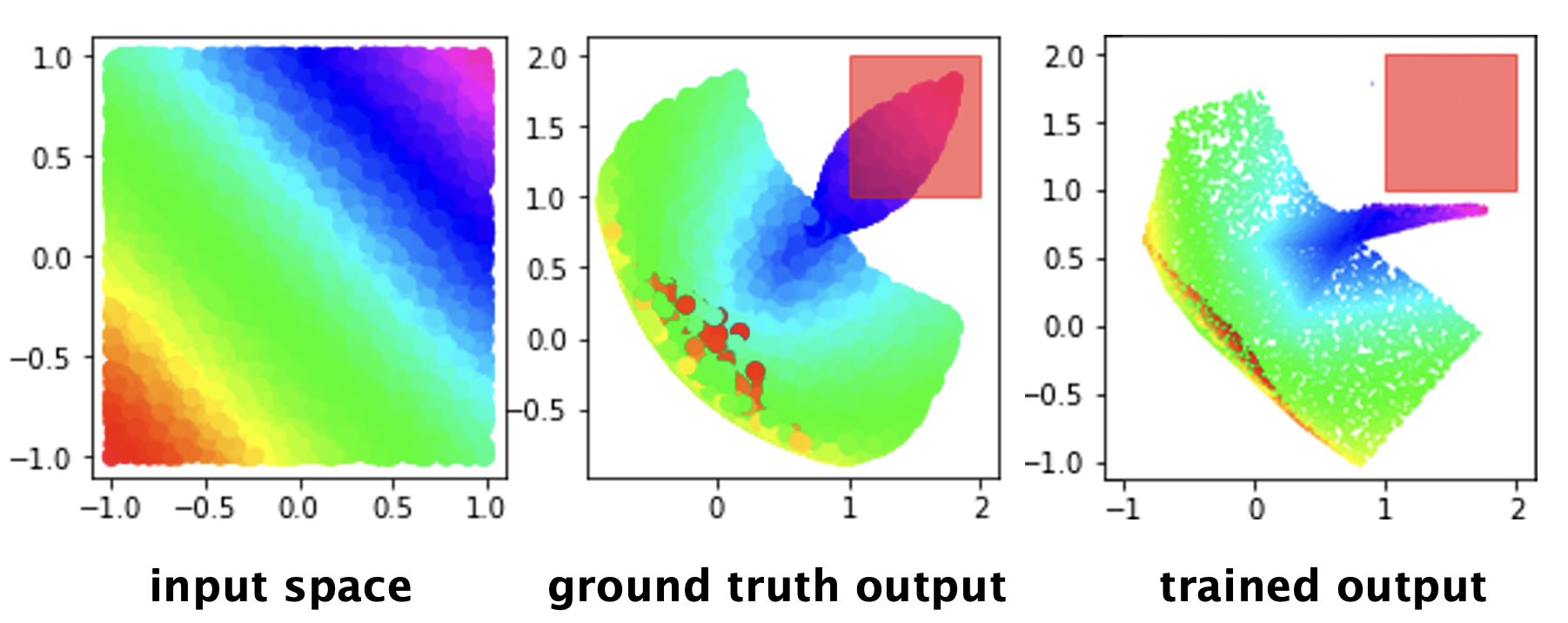}
    \caption{An example safe training result with the proposed method.
    The color gradient illustrates corresponding input and output points.
    With our method, the trained output does not intersect the unsafe set (red box).}
    \label{fig:method_overview}
    \vspace*{-0.4cm}
\end{figure}

\subsection{Contributions}
We propose a method to compute the output of a neural network with rectified linear unit (ReLU) activations given an input set represented as constrained zonotopes.
We then enforce performance by training under a differentiable zonotope intersection constraint, which guarantees safety upon convergence.
Our method is demonstrated on a small numerical example, and illustrated in Fig. \ref{fig:method_overview}.

\section{Preliminaries}\label{sec:preliminaries}

We now introduce our notation for neural networks and define constrained zonotopes.


In this work, we consider a fully-connected, ReLU-activated feedforward neural network $\nn{\cdot}: \inset \to \R^{n\idxdepth}$, with output $\xv\idxdepth = \nn{\xv\idxo}$ given an input $\xv\idxo \in \inset \subset \R^{n\idxo}$.
We call $\inset$ the input set.
We denote by $\depth \in \N$ the \defemph{depth} of the network and by $n\idxi$ the \defemph{width} of the $i$\ts{th} layer.
For each layer $k = 1,\cdots,\depth-1$, the hidden state of the neural network is given by
\begin{align}
    \xv\idx{k} = \relu{\linlayer{\xv\idx{k-1}, \Wt\idx{k-1}, \bias\idx{k-1}}},
\end{align}
where $\Wt\idx{k-1} \in \R^{n\idx{k-1}\times n\idx{k}}$, $\xv\idx{k}$ and $\bias\idx{k-1} \in \R^{n\idx{k}}$, and
\begin{align}
    \linlayer{\xv, \Wt, \bias} &= \Wt\xv + \bias,\\
    \relu{\xv} &= \max\{0,\xv\},
\end{align}
where $\linlayer{\cdot}$ is a linear layer operation, and $\relu{\cdot}$ is the ReLU nonlinearity with the max taken elementwise.
We do not apply the ReLU activation for the final output layer:
\begin{align}
    \xv\idxdepth = \linlayer{\xv\idxdepth, \Wt\idxdepth, \bias\idxdepth}.
\end{align}
The \defemph{reachable set} of the neural network is
\begin{align}\label{eq:nn_reach_set}
    \outset = \nn{\inset} \subset \R^{n\idxdepth}.
\end{align}

We represent the reachable set as a union of constrained zonotopes.
A \defemph{constrained zonotope} $\conzono{\ctr,\Gen,\Acon,\bcon} \subset \R^n$ is a set parameterized by a center $\ctr \in \R^n$, generator matrix $\Gen \in \R^{n\times \ngen}$, linear constraints $\Acon \in \R^{\ncon\times\ngen}$, $\bcon \in \R^\ncon$, and coefficients $\coef \in \R^\ngen$ as follows:
\begin{align}\label{eq:con_zono}
    \conzono{\ctr,\Gen,\Acon,\bcon} = \left\{\ctr + \Gen\coef\ |\ \norm{\coef}_\infty \leq 1,\ \Acon\coef = \bcon \right\}.
\end{align}

Importantly, the intersection of constrained zonotopes is also a constrained zonotope \cite[Proposition 1]{scott2016constrained}.
Let $Z_1 = \conzono{\ctr_1,\Gen_1,\Acon_1,\bcon_1}$ and $Z_2 = \conzono{\ctr_2,\Gen_2,\Acon_2,\bcon_2}$.
Then $Z_1 \cap Z_2$ is given by
\begin{align}\label{eq:conzono_intersection}
    Z_1 \cap Z_2 = \conzono{\ctr_1, [\Gen_1, \zeros], \begin{bmatrix}
            \Acon_1 & \zeros\\
            \zeros & \Acon_2\\
            \Gen_1 & -\Gen_2
        \end{bmatrix}, \begin{bmatrix}
            \bcon_1 \\
            \bcon_2 \\
            \ctr_2 - \ctr_1
        \end{bmatrix}}.
\end{align}
We leverage this property to evaluate constraints on the forward reachable set of our neural network.
\section{Method}\label{sec:method}

In this section, we first explain how to pass a constrained zonotope exactly through a ReLU nonlinearity; that is, we compute the reachable set of a ReLU activation given a constrained zonotope as the input.
We then discuss how to train a neural network using the reachable set to enforce constraints.
Finally, we explain how to compute the gradient of the constraint for backpropagation.

Before proceeding, we briefly mention that we can pass an input constrained zonotope $\Zin = \conzono{\ctr,\Gen,\Acon,\bcon}$ through a linear layer as
\begin{align}
    \linlayer{\Zin, \Wt, \bias} = \conzono{\Wt\ctr + \bias,\Wt\Gen,\Acon,\bcon}.
\end{align}
This follows from the definition in \eqref{eq:con_zono}.

\subsection{Constrained Zonotope ReLU Activation}

\begin{prop}\label{prop:conzono_relu_activation}
The ReLU activation of a constrained zonotope $\Zin \subset \R^n$ is:
\begin{align}
    \relu{\Zin} = \union_{i=1}^{2^n} \conzono{\ctr\idxi,\Gen\idxi,\Acon\idxi,\bcon\idxi},
\end{align}
where each output constrained zonotopes is given by:
\begin{subequations}\label{eq:relu_activation_params}
\begin{align}
    \Gen\idxi &= [\diag{\ntuple_i}\Gen,\ \zeros_{n\times n}],\\
    \Acon\idxi &= \begin{bmatrix}
            \Acon & \zeros_{\ncon\times n} \\
            \diag{\ones_{n\times 1} - 2\ntuple_i}\Gen & \diag{\vc{d}\idxi}
        \end{bmatrix},\\
    \ctr\idxi &= \diag{\ntuple_i}\ctr,\\
    \bcon\idxi &= \begin{bmatrix}
            \bcon \\
            -\diag{\ones_{n\times 1} - 2\ntuple_i}\ctr - \vc{d}\idxi
        \end{bmatrix},\regtext{and}\\
    \vc{d}\idxi &= \tfrac{1}{2}\left(\Gen_+\ones_{\ngen\times 1} - \diag{\ones_{n\times 1} - 2\ntuple_i}\ctr\right),
\end{align}
\end{subequations}
where $\Gen_+ \in \R^{n\times\ngen}$ is a matrix containing the elementwise absolute value of $\Gen$ and $\ntuple_i \in \R^{n\times1}$ is the $i$\ts{th} combination of the $2^n$ possible $n$-tuples defined over the set $\left\{0,1\right\}^n$.
\end{prop}
\begin{proof}
The formulation in \eqref{eq:relu_activation_params} follows from treating the $\max$ operation applied to all negative elements of the input zonotope as a sequence of two operations.
\textit{First}, we intersect the input constrained zonotope with the halfspace defined by the vector $\ntuple_i$ in the codomain of $\relu{\cdot}$; this is why the linear operator $\diag{\ntuple_i}$ is applied to each $\Gen\idxi$ and $\ctr\idxi$, as given by the analytical intersection of a constrained zonotope with a halfspace \cite[Eq. 10]{raghuraman2020set}.
\textit{Second}, we zero out the dimension corresponding to that halfspace/unit vector (i.e., project all negative points to zero).
Since the max is taken elementwise, there are $2^n$ possible intersection/zeroings when considering each dimension as either activated or not.
\end{proof}
\noindent Proposition \ref{prop:conzono_relu_activation} is illustrated in Fig. \ref{fig:conzon_relu_example}.

Per Proposition \ref{prop:conzono_relu_activation}, passing a constrained zonotope through a ReLU nonlinearity produces a set of $2^n$ constrained zonotopes.
A similar phenomenon is found in ReLU activations of other set representations \cite{tran2020verification}, with exponential growth in the computational time and memory required as a function of layer width and number of layers.
To mitigate this growth, empty constrained zonotopes can be pruned after each activation, hence our next discussion.

\begin{figure}[ht]
    \centering
    \includegraphics[width=0.7\columnwidth]{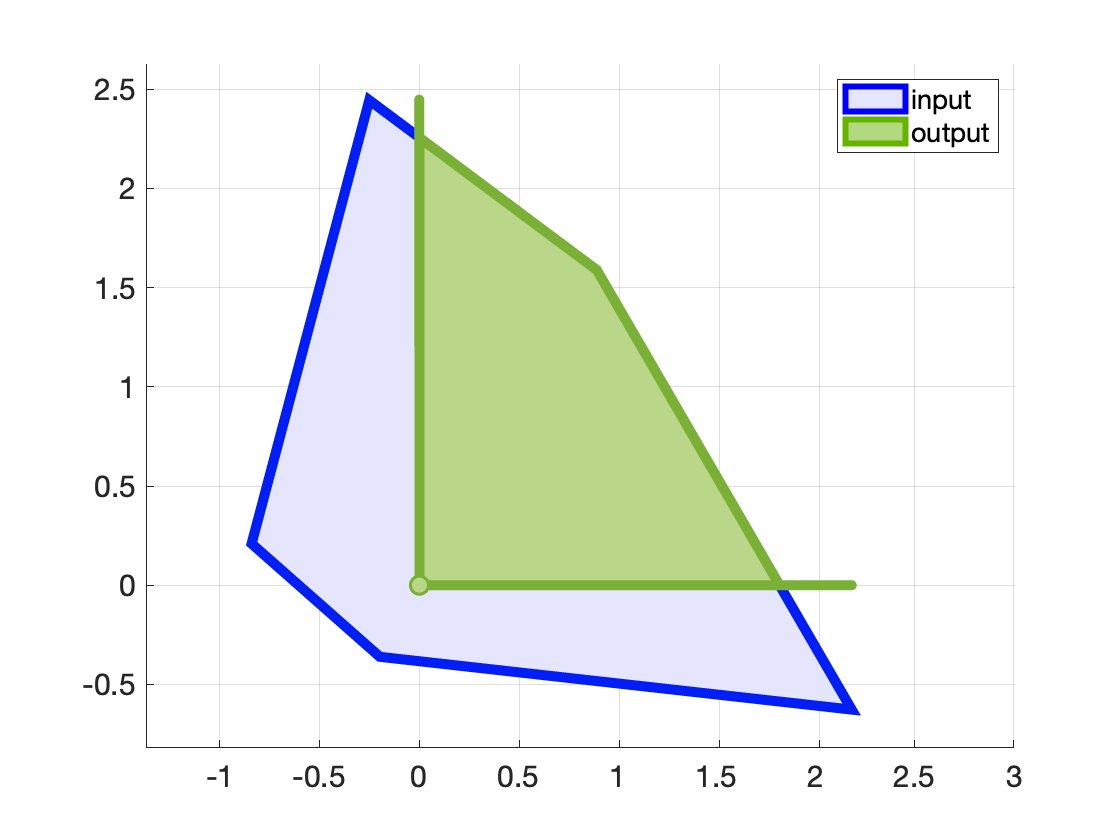}
    \caption{An illustration of passing a constrained zonotope (blue) through a 2-D ReLU nonlinearity, resulting in the green output set, which is the union of 4 constrained zonotopes.}
    \label{fig:conzon_relu_example}
    \vspace*{-0.5cm}
\end{figure}

\subsection{Constrained Zonotope Emptiness Check}

To check if $\Zin = \conzono{\ctr,\Gen,\Acon,\bcon}$ is empty, we solve a linear program (LP) \cite[Proposition 2]{scott2016constrained}:
\label{prog:emptiness_check}
\begin{align}
    \min_{\coef, \colval} \left\{\colval\ |\ \Acon\coef = \bcon\ \regtext{and}\ \norm{\coef}_\infty \leq \colval\right\}.
\end{align}
Then, $\Zin$ is empty if and only if $\colval > 1$.
Importantly, by construction, as long as there exist feasible $\coef$ (for which $\Acon\coef = \bcon$), then \eqref{prog:emptiness_check} is \textit{always} feasible.
Since the intersection of constrained zonotopes is also a constrained zonotope as in \eqref{eq:conzono_intersection}, we can use this emptiness check to enforce collision-avoidance (i.e., non-intersection) constraints.
This is the basis of our constrained training method.


\subsection{Constrained Neural Network Training}

The main goal of this paper is constrained neural network training.
For robotics in particular, as future work, our goal is to train a robust controller.
In this work, we consider an unsafe output set which could represent, e.g., actuator limits or obstacles in a robot's workspace (in which case the output of the neural network is passed through a robot's dynamics).

\subsubsection{Generic Formulation}
Consider an input set $\inset \subset \R^{n\idxo}$ represented by a constrained zonotope, an unsafe set $\usset \subset \R^{n\idxdepth}$, a training dataset $(\xv\idxo_j, \yv_j)$, $j = 1, \cdots, m$, $\yv_j \in \R^{n\idxdepth}$ of training examples $\xv\idxo_j$ and labels $\yv_j$, and an objective loss function $\Lobj{\xv\idxo_1, \cdots, \xv\idxo_m, \yv_1, \cdots, \yv_m}$.
Let $\weights = \{\Wt\idxo, \cdots, \Wt\idxdepth\}$ be the collection of all of the neural network weights and $\biases = \{\bias\idxo, \cdots, \bias\idxdepth\}$ all the biases.
We formulate the training problem as:
\begin{subequations}\label{prog:training_problem}
\begin{align}
    \min_{\weights,\biases}\quad
    &\Lobj{\xv\idxo_1, \cdots, \xv\idxo_m, \yv_1, \cdots, \yv_m}, \\
    \regtext{s.t.}\quad
    &\outset \cap \usset = \emptyset,
\end{align}
\end{subequations}
where $\outset$ is the reachable set as in \eqref{eq:nn_reach_set}.
We write the loss as a function of all of the input/output data (as opposed to batching the data) for ease of presentation.

\subsubsection{Set and Constraint Representations}

We represent the input set and unsafe set as constrained zonotopes, $\inset = \inzono$ and $\usset = \uszono$.
Similarly, it follows from Proposition \ref{prop:conzono_relu_activation} that the output set $\outset$ can be exactly represented as a union of constrained zonotopes:
\begin{align}
    \outset = \union_{i=1}^{\nout} \outzono_i,
\end{align}
where $\nout$ depends on the layer widths and network depth.

Recall that $\outzono_i \cap \uszono$ is a constrained zonotope as in \eqref{eq:conzono_intersection}.
So, to compute the constraint loss, we evaluate $\outset \cap \usset$ by solving \eqref{prog:emptiness_check} for each constrained zonotope $\outzono_i \cap \uszono$ with $i = 1, \cdots$.
Then, denoting $\colval\opt_i$ as the output of \eqref{prog:emptiness_check} for each $\outzono_i \cap \uszono$, we represent the constraint $\outset \cap \usset = \emptyset$ as a function $\Lcon{\cdot}$ for which
\begin{align}\label{eq:constraint_loss}
    \Lcon{\outzono_i,\uszono} = 1 - \colval\opt_i,
\end{align}
which is negative when feasible as is standard in constrained optimization \cite{nocedal2006numerical}.
Using \eqref{eq:constraint_loss}, we ensure the neural network obeys constraints by checking $\Lcon{\outzono_i,\uszono} < 0$ for each $i = 1,\cdots,\nout$.



\subsection{Differentiating the Collision Check Loss}

To train using backpropagation, we must differentiate the constraint loss $\Lcon{\cdot}$.
This means we must compute the gradient of \eqref{prog:emptiness_check} with respect to the problem parameters $\Acon$ and $\bcon$, which are defined by the centers, generators, and constraints of the output constrained zonotope set.
To do so, we leverage techniques from \cite{amos2017optnet}, which can be applied because \eqref{prog:emptiness_check} is always feasible.

Consider the Lagrangian of \eqref{prog:emptiness_check}:
\begin{align}
    J(\coef,v,\slackineq,\slackeq) = q\trans(\coef,v) + \Acon\trans\slackeq + \slackineq\trans(\vc{G}(\coef,v) - \vc{g}),
\end{align}
where $\slackineq$ is the dual variable for the inequality constraint and $\slackeq$ is the dual variable for the equality constraint.
For any optimizer $(\coef\opt,v\opt,\slackineq\opt,\slackeq\opt)$, the optimality conditions are
\begin{subequations}\label{eq:KKT_conds_for_coll_check_LP}
\begin{align}
    \vc{q}\trans(\coef\opt,v\opt) + \Acon\trans\slackeq\opt + \vc{G}\trans\slackineq\opt &= 0, \\
    \Acon(\coef\opt,v\opt)\opt - \bcon &= 0,\ \regtext{and}\\
    \diag{\slackineq\opt}\vc{G}(\coef\opt,v\opt)\opt &= 0,
\end{align}
\end{subequations}
where we have used the fact that $\vc{g} = \zeros$.
Taking the differential (denoted by $d$) of \eqref{eq:KKT_conds_for_coll_check_LP}, we get
\begin{align}\begin{split}
\label{eq:differential_of_opt_conds}
    \begin{bmatrix}
        \zeros &\vc{G}\trans & \Acon\trans \\
        \diag{\slackineq\opt} & \diag{\vc{G}(\coef\opt,v\opt)\opt} & \zeros \\
        \Acon & \zeros & \zeros
    \end{bmatrix}\begin{bmatrix}
        d(\coef\opt,v\opt) \\
        d\slackineq \\
        d\slackeq
    \end{bmatrix} =\\
    \begin{bmatrix}
        -d\vc{q} - d\Acon\trans\slackeq\opt \\
        -\diag{\slackineq\opt}d\vc{G}(\coef\opt,v\opt) \\
        -d\Acon(\coef\opt,v\opt)\opt + d\bcon
    \end{bmatrix},
\end{split}\end{align}
We can then solve \eqref{eq:differential_of_opt_conds} for the Jacobian of $v\opt$ with respect to any entry of the zonotope centers or generators by setting the right-hand side appropriately (see \cite{amos2017optnet} for details).
That is, we can now differentiate \eqref{eq:constraint_loss} with respect to the elements of $\ctr_1$, $\Gen_1$, $\ctr_2$, or $\Gen_2$.
In practice, we differentiate \eqref{prog:emptiness_check} automatically using the \texttt{cvxpylayers} library \cite{agrawal2019differentiable}.
\section{Numerical Example}\label{sec:numerical_example}

We test our method by training a 2-layer feedforward ReLU network with input dimension 2, hidden layer size of 10, and output dimension of 2.
We chose this network with only one ReLU nonlinearity layer, as recent results have shown that a shallow ReLU network performs similarly to a deep ReLU network with the same amount of neurons \cite{hanin2019deep}.
However, note that our method (in particular Proposition \ref{prop:conzono_relu_activation}) does generalize to deeper networks.
We pose this preliminary example as a first effort towards this novel style of training.

\textit{Problem Setup.}
We seek to approximate the function
\begin{align}
    f(x_1, x_2) &= \begin{bmatrix}
            x_1^2 + \sin{x_2}\\
            x_2^2 + \sin{x_1}
        \end{bmatrix},
\end{align}
with $\inset = \conzono{\zeros_{2\times1},\eye_2,\emptyset,\emptyset}$.
We create an unsafe set in the output space as
\begin{align}
    \usset = \conzono{\begin{bmatrix}1.5\\1.5\end{bmatrix},0.5\eye_2,\emptyset,\emptyset}.
\end{align}
The training dataset was generated as $\ndata = 10^4$ random input/output pairs $(\xv\idxo_j, \yv_j)$ with $\yv_j = f(\xv\idxo_j)$ by sampling uniformly in $\inset$.
The unsafe set $\usset$ and each $\xv\idxo_j$ and $\yv_j$ are plotted in Figs. \ref{fig:Unconstrained_Training} and \ref{fig:Constrained_Training}.
The objective loss is
\begin{align}
    \Lobj{\cdot} = \dfrac{1}{\ndata}\sum_{j=1}^{\ndata}\norm{\nn{\xv\idxo_j} - \yv_j}_2^2
\end{align}
where we use $\cdot$ for concision in place of the training data.


\textit{Implementation.}
We implemented our method\footnote{Our code is available online: \url{https://github.com/Stanford-NavLab/constrained-nn-training}} in PyTorch \cite{paszke2019pytorch} with \texttt{optim.SGD} as our optimizer on a desktop computer with 6 cores, 32 GB RAM, and an RTX 2060 GPU.

We trained the network for $10^3$ iterations with and without the constraints enforced.
To enforce hard constraints as in  \eqref{prog:training_problem}, in each iteration, we compute the objective loss function across the entire dataset, then backpropagate the objective gradient; then, we compute the constraint loss as in \eqref{eq:constraint_loss} for all active constraints, then backpropagate.
For future work we will apply more sophisticated constrained optimization techniques (e.g., an active set method) \cite[Ch. 15]{nocedal2006numerical}.

With our current na\"ive implementation, constrained training took approximately 5 hours, whereas the unconstrained training took 0.5 s.
Our method is slower due to the need to compute an exponentially-growing number of constrained zonotopes as in Proposition \ref{prop:conzono_relu_activation}.
However, we notice that the GPU utilization is only 1-5\% (the reachability propagation is not fully parallelized), indicating significant room for increased parallelization and speed.

\textit{Results and Discussion.}
Results for unconstrained and constrained training are shown in Fig. \ref{fig:Unconstrained_Training} and Fig. \ref{fig:Constrained_Training}.
Our proposed method avoids the unsafe set.
Note the output constrained zonotopes (computed for both networks) contain the colored output points, verifying our exact set representation in Proposition \ref{prop:conzono_relu_activation}.

Table \ref{tab:Quan_Results} shows results for unconstrained and constrained training; importantly, our method obeys the constraints.
As expected for nonlinear constrained optimization, the network converged to a local minimum while obeying the constraints.
The key challenge is that the constrained training is several orders of magnitude slower than unconstrained training.
We plan to address in future work by increased parallelization, by pruning of our reachable sets \cite{tran2020nnv}, and by using anytime verification techniques \cite{vincent2020reachable}.

\begin{figure}
    \centering
    \includegraphics[width=0.8\columnwidth]{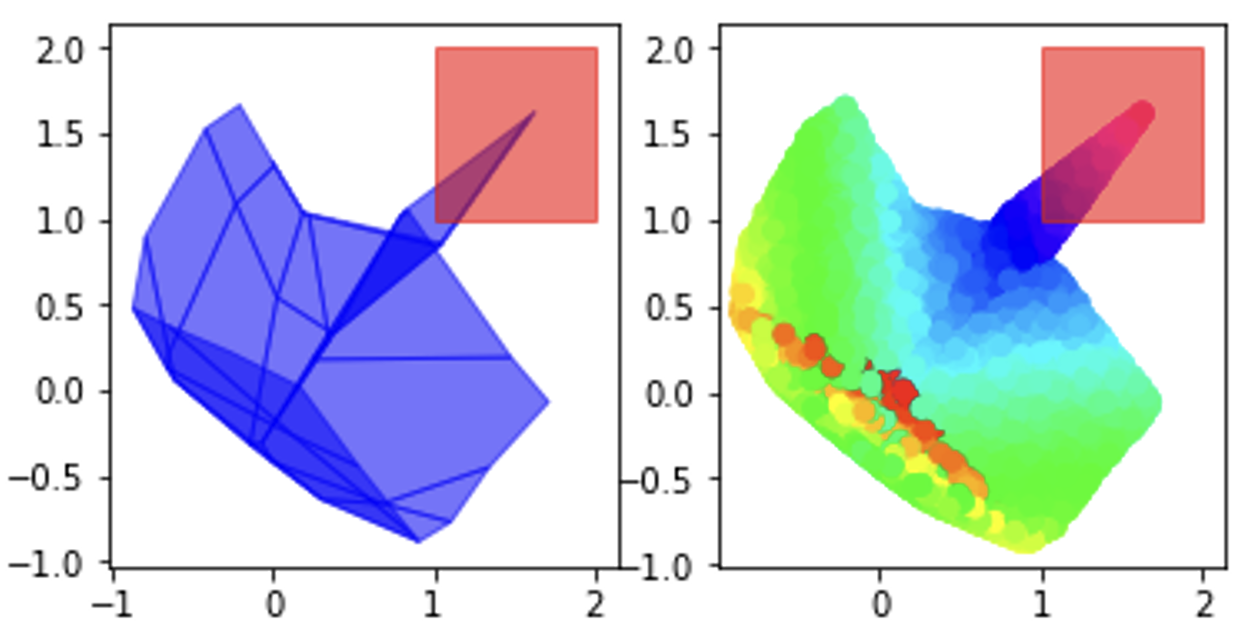}
    \caption{Unconstrained training, with $\outset$ plotted on the left and each $\nn{\xv\idxo_j}$ plotted on the right. The output approximates the function well but does not avoid the unsafe space.}
    \label{fig:Unconstrained_Training}
\end{figure}

\begin{figure}
    \centering
    \includegraphics[width=0.8\columnwidth]{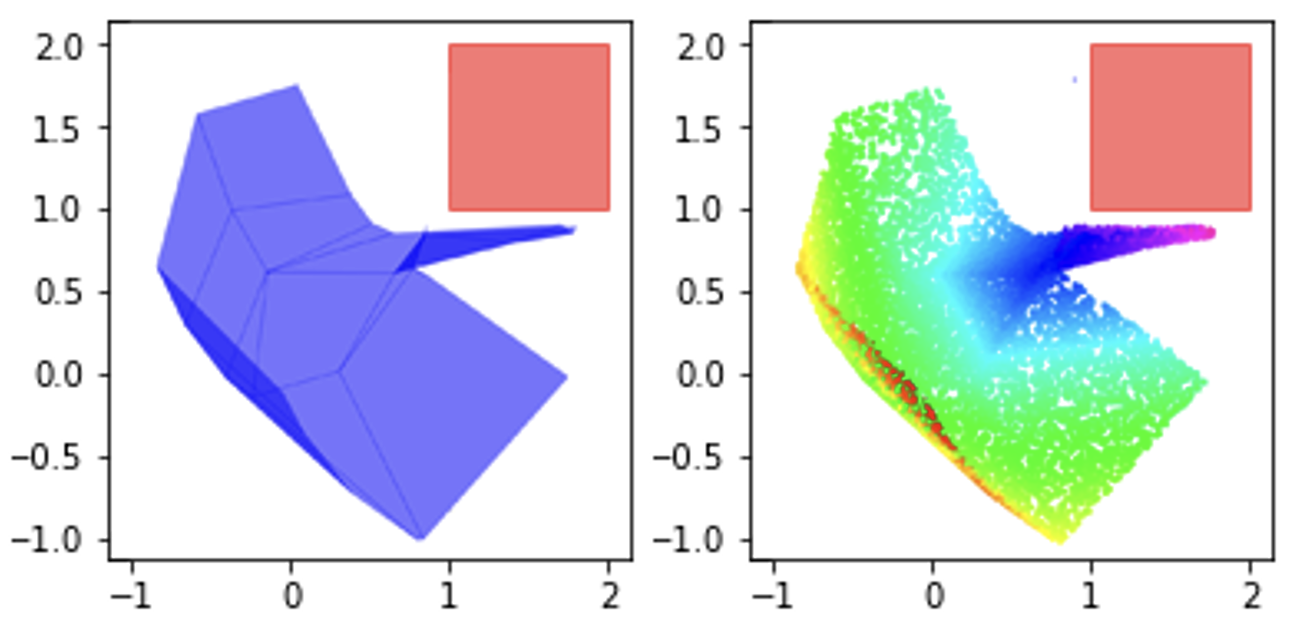}
    \caption{Constrained training. 
    The output approximates the function while avoiding the unsafe space.}
    \label{fig:Constrained_Training}
    \vspace*{-0.5cm}
\end{figure}

\begin{table}[ht]
    \centering
    \begin{tabular}{c|c c c}
     & Unconstrained & Constrained \\
     \hline
     final objective loss & \textbf{0.0039} & 0.0127 \\
     final constraint loss & 0.0575 &\textbf{ 0.0000}
    \end{tabular}
    \caption{}
    \label{tab:Quan_Results}
    \vspace*{-0.5cm}
\end{table}

\section{Conclusion and Future Work}\label{sec:conclusion}

This work proposes a constrained training method for feedforward ReLU neural networks.
We demonstrated the method successfully on a small example of nonlinear function approximation.
Given the ability to enforce output constraints, the technique can potentially be applied to offline training for safety-critical neural networks.

Our current implementation has several drawbacks to be addressed in future work.
First, the method suffers an exponential blowup of constrained zonotopes through a ReLU.
We hope to improve the forward pass step by using techniques such as \cite{vincent2020reachable} instead of layer-by-layer evaluation to compute the output set, and by conservatively estimating the reachable set similar to \cite{rubies2019fast}.
We also plan to apply the method on larger networks, such as for autonomous driving in \cite{tran2020nnv} or the ACAS Xu network \cite{kochenderfer2011robust, kochenderfer2012next, kochenderfer2015optimized} for aircraft collision avoidance.
In general, our goal is to train robust controllers where the output of a neural network must obey actuator limits and obstacle avoidance (for which the network output is passed through dynamics).

\renewcommand{\bibfont}{\normalfont\footnotesize}
{\renewcommand{\markboth}[2]{}
\printbibliography}

\end{document}